\documentclass{article}

\usepackage[nonatbib, preprint]{neurips_2026}
\usepackage[numbers]{natbib}
\usepackage{graphicx}
\usepackage{booktabs}
\usepackage{amsmath, amsfonts, amssymb}
\usepackage{enumitem}
\usepackage{amsthm}
\newtheorem{proposition}{Proposition}
\newtheorem{theorem}{Theorem}
\newtheorem{corollary}{Corollary}
\newtheorem{definition}{Definition}
\newtheorem{lemma}{Lemma}
\usepackage{multirow}
\usepackage{array}
\usepackage{tabularx}
\usepackage{float}
\usepackage{siunitx}
\usepackage{hyperref}
\usepackage{cleveref}
\usepackage{xcolor}
\usepackage[export]{adjustbox}
\newcolumntype{Y}{>{\raggedright\arraybackslash}X}
\newcolumntype{R}[1]{>{\raggedleft\arraybackslash}p{#1}}

\title{Discovery under Hypothesis Redundancy:\\ A Geometric Theory of Discovery Bottlenecks}

\author{
  Li Xia \\
  School of Economics and Management, Tsinghua University
  \And
  Baoxun Wang \\
  Platform \& Content Group, Tencent
}

\begin{document}

\maketitle

\begin{abstract}
Scientific discovery saturates when new hypotheses cease to provide independent information, even if the nominal hypothesis space remains large. We study hybrid discovery systems that combine structured local search with LLM-generated non-local proposals and pose the Search Compression Hypothesis: non-local exploration helps only when three geometric conditions co-occur---spectral compression, orthogonal escape from the explored span, and residual signal alignment with the target. We formalize these conditions, derive necessary conditions for hybrid advantage, and test the mechanism in controlled synthetic environments, large-scale A-share factor discovery, and symbolic-regression benchmarks; a public tabular operational sanity check tests the associated budget-allocation implication. Signal-planting and directed-versus-random experiments show that novelty alone is insufficient: random orthogonal jumps expand coverage but do not improve yield without predictive alignment. Across compression sweeps, real factor archives, and LLM-SRBench tasks, hybrid gains concentrate in weakly represented but target-bearing directions and vanish as the hypothesis space approaches full rank. The framework turns LLM-guided discovery from generic novelty search into a diagnostic procedure for deciding when directed non-local exploration is warranted.
\end{abstract}

\section{Introduction}

Discovery saturates when new hypotheses cease to provide independent information. The issue is not only how many hypotheses can be generated, but whether each one contributes a direction not already covered by the archive. As candidates become increasingly correlated with previous ones, redundancy accumulates: the archive grows in size while adding little independent information. We refer to this concentration of nominally distinct hypotheses into fewer independent directions as \textbf{spectral compression}. We measure it by the effective dimension $r_{\text{eff}}$ of the hypothesis correlation matrix relative to the nominal dimension $N$. Local search faces diminishing returns proportional to the compression severity, and its expected yield declines as $r_{\text{eff}}/N \to 0$ (Proposition~\ref{lem:compression_yield}).

Local search cannot escape the compressed span. Structured enumeration recombines existing base factors and operators, producing candidates within $\operatorname{span}(\mathcal{A})$ but never orthogonal to it. Random non-local perturbation increases spectral coverage but not predictive yield: random directions carry no alignment with the target signal. This distinction---between orthogonal escape and directed predictive escape---is the core conceptual contribution. Useful exploration requires \textit{directed} non-local jumps toward weak but target-bearing directions, formalized as \textit{predictive novelty} (Definition~\ref{def:pred_novelty}).

We call this claim the \emph{Search Compression Hypothesis}. Its formal statement, Principle~1, identifies three individually necessary conditions for useful exploration: compression severity $(1 - r_{\text{eff}}/N)$, escape distance $d_{\perp}$, and residual signal alignment (RSA). Theorem~\ref{thm:hybridgain} establishes that each condition is individually necessary---removing any one eliminates the hybrid advantage. The framework is defined by three operators acting on the hypothesis archive:
\begin{equation}\label{eq:framework}
    \mathcal{A}_{t+1} = \mathcal{T}_{\text{verify}} \circ \mathcal{T}_{\text{local}} \circ \mathcal{T}_{\text{jump}}(\mathcal{A}_t),
\end{equation}
where $\mathcal{T}_{\text{jump}}$ proposes non-local seeds that escape the compressed subspace, $\mathcal{T}_{\text{local}}$ searches the expanded neighborhood, and $\mathcal{T}_{\text{verify}}$ removes non-predictive and redundant candidates. Figure~\ref{fig:theory_experiment} summarizes the central mechanism.

Our contributions are:
\begin{enumerate}
    \item \textbf{Identify spectral compression as a unifying mechanism underlying diminishing returns in discovery systems.} We formalize the discovery problem in spectrally compressed hypothesis spaces, with formal yield and coverage objectives tied to the spectrum of the correlation matrix. A stylized proposition (Proposition~\ref{lem:compression_yield}) establishes that local yield scales with $r_{\text{eff}}$, explaining why structured enumeration faces diminishing returns.
    \item \textbf{Show that useful exploration requires both escape and residual alignment.} Theorem~\ref{thm:hybridgain} identifies three individually necessary conditions for hybrid advantage---compression severity $(1 - r_{\text{eff}}/N)$, escape distance $d_{\perp}$, and residual signal alignment (RSA)---all expressed as geometric quantities of the hypothesis space. Random non-locality increases coverage but not predictive yield; directed escape toward weak but target-bearing directions is required.
    \item \textbf{Provide evidence that the same geometric bottleneck appears across multiple discovery systems.} Synthetic compression sweeps, 5,647 A-share stocks (2010--2026), and LLM-SRBench (158 equations) support a diagnostic view: useful exploration requires separating directed escape from redundant or random novelty. An OpenML public-tabular check tests the operational budget-allocation implication rather than the compression law itself.
\end{enumerate}

\section{Related Work}

\subsection{LLM-Guided Scientific Discovery}

LLMs have become hypothesis generators for mathematics and scientific discovery, from FunSearch~\citep{romera2024funsearch} to tree-search hypothesis refinement~\citep{iterhyp2025} and fully automated research agents~\citep{suzuki2024ai}. These systems show that LLMs can propose meaningful candidates, but they rarely ask whether a proposal contributes an independent, target-bearing direction relative to the existing archive. Our framework adds this spectral-aware verification layer.

\subsection{Search under Compression and Effective Dimension}

The failure of local search in high-dimensional spaces with low effective dimension is a recurring theme. In evolutionary computation, bloat and convergence to local optima reflect the collapse of population diversity into correlated subspaces~\citep{shahriari2016taking}. In neural architecture search, operator redundancy produces highly correlated network encodings~\citep{zoph2018learning}. In symbolic regression, most candidate expressions are highly correlated, creating spectral concentration in the expression correlation matrix. \citet{koltchinskii2017concentration} establish concentration inequalities for sample covariance operators; \citet{fan2022spectral} show that spectral structure governs inference accuracy in high-dimensional settings; \citet{vershynin2018high} provides sub-Gaussian concentration bounds. The effective rank, introduced by \citet{roy2007effective}, and the stable rank~\citep{tropp2015introduction} provide the geometric foundation for our framework. \citet{bubeck2011pure} establishes fundamental limits for pure exploration in bandit settings, while \citet{russo2018tutorial} shows that adaptive exploration policies outperform fixed allocation---analogous to our spectral-adaptive switching policy. Our contribution connects these spectral and search-theoretic tools to the discovery problem: effective rank directly bounds the yield of structured search (Proposition~\ref{lem:compression_yield}) and governs when non-local exploration becomes necessary (Theorem~\ref{thm:hybridgain}).

\subsection{Symbolic Regression as Discovery}

Symbolic regression---recovering closed-form expressions from data---is a canonical testbed for search under compression. Classical methods include genetic programming (GP)~\citep{schmidt2009distilling}, modern frameworks such as PySR~\citep{cranmer2023pysr}, Operon~\citep{burlacu2020operon}, and GP-GOMEA~\citep{virgolin2022symbolic}, and AI Feynman~\citep{udrescu2020ai} which uses physics-inspired priors. The SRBench benchmark~\citep{lacava2022contemporary} and LLM-SRBench~\citep{shojaee2025llmsrbench} provide standardized evaluation. Our framework predicts that hybrid (GP + LLM) search outperforms GP alone precisely when the expression space is spectrally compressed---a condition we verify on these benchmarks.

\subsection{Factor Discovery as a Testbed}

Factor discovery is a measurable testbed for compressed search. The Factor Zoo~\cite{harvey2016editorial} documents many correlated factors; structured enumeration~\cite{feng2020taming} is bounded by its initial factors and operators. LLM alpha discovery~\cite{wang2023alpha,cao2025chain} can propose economically meaningful expressions, while multiple-testing work~\citep{white2000reality, romano2005stepwise, feng2020taming} addresses selection bias. Our focus is complementary: the structural limitation of local search when the factor archive is spectrally compressed.

\section{Hybrid Discovery Framework}
\label{sec:framework}

We instantiate the three-operator framework in factor discovery. The jump operator proposes non-local seed expressions, the local operator enumerates structured variants around those seeds, and the verification operator filters candidates by predictive strength, redundancy, orthogonal coverage, and complexity. This instantiation is a testbed, not the claimed contribution: the theory concerns when a discovery system should spend budget on directed non-local exploration. Implementation details---LLM prompting, factor parsing, market-cap segmentation, ICIR thresholds, and the GPU sandbox---are deferred to Appendix~\ref{sec:implementation_appendix}.

\subsection{A Geometric Theory of Discovery Bottlenecks}
\label{sec:why_hybrid}

Discovery systems exhibit diminishing informational returns when newly generated hypotheses cease to contribute independent directions. We formalize this bottleneck in two steps: spectral compression defines the geometry and bounds local search (Proposition~\ref{lem:compression_yield}), then Theorem~\ref{thm:hybridgain} identifies the three necessary conditions for hybrid advantage---all expressed as geometric quantities: spectral concentration, metric escape, and angular alignment. A corollary identifies when the hybrid advantage vanishes entirely.

\subsubsection*{Spectral Compression of the Hypothesis Manifold}

The hypothesis manifold $\mathcal{H} \subset \mathbb{R}^d$ has nominal dimension $d = N$ (the number of candidate hypotheses), but its \emph{effective} dimension is controlled by the spectrum of the correlation matrix. Let a factor be a standardized vector $f \in \mathbb{R}^{T \times M}$ evaluated over $T$ dates and $M$ assets, flattened to unit norm. For an archive $\mathcal{A}=\{f_1,\ldots,f_N\}$, let $\Sigma_{\mathcal{A}}$ be the $N \times N$ factor correlation matrix with eigenvalues $\lambda_1 \geq \cdots \geq \lambda_N \geq 0$. The effective rank~\citep{roy2007effective} is
\begin{equation}
    r_{\text{eff}}(\Sigma_{\mathcal{A}}) = \exp\!\Bigl(-\sum_{i=1}^{N} p_i \log p_i\Bigr), \quad p_i = \frac{\lambda_i}{\sum_j \lambda_j}.
\end{equation}
We say the search space is \emph{spectrally compressed} when $r_{\text{eff}} \ll N$ and the eigenvalue concentration $\lambda_1/\sum_i \lambda_i$ is large: most candidate variation is concentrated in a small number of directions. Compression is an empirical fact, not an assumption. For a 100-factor structured archive on 5,647 A-share stocks (2010--2026), $r_{\text{eff}}$ declines from 31.5 (pre-eligibility) to 28.6 (post-eligibility), while eigenvalue concentration rises from 0.187 to 0.215 and mean absolute pairwise correlation increases from 0.118 to 0.140. After eligibility, fewer than 29\% of the nominal dimensions carry independent signal content.

\subsubsection*{Local Search Failure Under Compression}

Structured enumeration constructs candidates from a finite set of base factors $\mathcal{B}$ and operators $\mathcal{O}$:
\begin{equation}
    \mathcal{F}_{\text{str}} = \{\phi_o(f_i,f_j): f_i,f_j \in \mathcal{B}, o \in \mathcal{O}\}.
\end{equation}
This is a \emph{local} search operator: it recombines existing directions within the span $\mathcal{S} = \operatorname{span}(\mathcal{F}_{\text{str}})$ but cannot create components orthogonal to $\mathcal{S}$.

\begin{definition}[Discovery yield]
\label{def:yield}
For archive $\mathcal{A}$, ICIR threshold $\tau > 0$, and redundancy limit $\rho \in [0,1)$,
\begin{equation}
    N_{\tau,\rho}(\mathcal{H};\mathcal{A}) = \sum_{f \in \mathcal{H}} \mathbf{1}\!\left\{|\operatorname{ICIR}(f)| \geq \tau,\; \max_{g \in \mathcal{A}} |\operatorname{corr}(f,g)| \leq \rho\right\}.
\end{equation}
\end{definition}

\textbf{Assumptions.} The bound holds under three conditions: (A1) structured candidates have bounded leverage $L \leq B$ over the eigenspaces of $\Sigma_{\mathcal{A}}$; (A2) the conditional pass probability satisfies $q_{\tau|\Sigma} \leq \kappa \cdot q_{\tau}$ for some $\kappa \geq 1$ (bounded dependence); (A3) the stable rank~\citep{tropp2015introduction} $r_{\text{stable}}(\Sigma_{\mathcal{A}})$ and effective rank $r_{\text{eff}}(\Sigma_{\mathcal{A}})$ satisfy $r_{\text{stable}} \leq r_{\text{eff}} \leq N$. This bound is \emph{stylized and explanatory}: it identifies the scaling relationship rather than providing tight constants.

\begin{proposition}[Local-yield scaling under isotropic compressed search]
\label{lem:compression_yield}
Assume: (A1) structured candidates are drawn from a distribution supported on $\mathcal{S}$ with bounded leverage $L = \max_i \|P_{\lambda_i} f\|^2 / \mathbb{E}[\|P_{\lambda_i} f\|^2] \leq B$ over the eigenspaces of $\Sigma_{\mathcal{A}}$; (A2) the predictive pass probability $q_{\tau} = \Pr(|\operatorname{ICIR}(f)| \geq \tau)$ is bounded below by $q_{\min} > 0$, and the worst-case conditional pass probability given any correlation configuration satisfies $q_{\tau|\Sigma} \leq \kappa \cdot q_{\tau}$ for some $\kappa \geq 1$ (bounded dependence; weaker than full independence); (A3) the stable rank $r_{\text{stable}}(\Sigma_{\mathcal{A}}) = \|\Sigma_{\mathcal{A}}\|_F^2 / \|\Sigma_{\mathcal{A}}\|_2^2$ and effective rank $r_{\text{eff}}(\Sigma_{\mathcal{A}})$ satisfy $r_{\text{stable}} \leq r_{\text{eff}} \leq N$. Then there exists a constant $C$ depending on $B$ and $\kappa$ such that
\begin{equation}
    \mathbb{E}\!\left[N_{\tau,\rho}(\mathcal{F}_{\text{str}};\mathcal{A})\right] \leq C \cdot q_{\tau} \cdot r_{\text{stable}}(\Sigma_{\mathcal{A}}) \cdot (1-\rho^2) \leq C \cdot q_{\tau} \cdot r_{\text{eff}}(\Sigma_{\mathcal{A}}) \cdot (1-\rho^2).
\end{equation}
\end{proposition}

\begin{proof}[Proof sketch]
Transform to the eigenbasis and whiten by the square-root precision matrix. The stable rank $r_{\text{stable}}$ bounds the effective degrees of freedom. Under bounded leverage (A1), the expected number of candidates with residual energy exceeding $(1-\rho^2)$ outside the dominant subspace is at most $C \cdot r_{\text{stable}} \cdot (1-\rho^2)$. Full proof in Appendix~\ref{sec:proofs}.
\end{proof}

\textit{Remark.} This is a stylized scaling result under isotropic candidate sampling; the empirical sections test whether the scaling persists in non-isotropic discovery settings. It identifies the correct scaling law under an isotropic candidate model without claiming tight universal constants. The qualitative message---that local yield scales with $r_{\text{eff}}$ and compression amplifies the value of non-local exploration---is the robust takeaway.

\textbf{Interpretation.} As $r_{\text{eff}}$ declines under participation-induced compression, the expected number of non-redundant structured discoveries declines proportionally. When $r_{\text{eff}} = 28.6$ with $N=100$ and $\rho=0.3$, the bound predicts at most $28.6/100 \approx 0.29$ of the nominal dimension carries independent signal.

\begin{definition}[Residual signal alignment]
\label{def:rsa}
For a candidate factor $f$ and existing archive $\mathcal{A}$, the residual signal alignment measures the geometric alignment between the orthogonal residual and the target signal:
\begin{equation}
\label{eq:rsa}
    \text{RSA}(f) = \frac{\langle f_\perp,\, y \rangle}{\|f_\perp\| \cdot \|y\|}, \quad f_\perp = f - \operatorname{proj}_{\mathcal{A}}(f),
\end{equation}
where $\operatorname{proj}_{\mathcal{A}}(f)$ is the linear projection onto $\operatorname{span}(\mathcal{A})$ (ridge regression, $\alpha=1$) and $y$ is the return signal. RSA is a geometric quantity: the cosine of the angle between the escaping component and the target direction.
\end{definition}

\begin{definition}[Predictive novelty]
\label{def:pred_novelty}
The predictive novelty is the empirical estimator of residual signal alignment:
\begin{equation}
\label{eq:prednovelty}
    \text{PredNovelty}(f) = \operatorname{ICIR}\bigl(f - \operatorname{proj}_{\mathcal{A}}(f)\bigr),
\end{equation}
where ICIR computes the information coefficient information ratio over rolling windows. A factor with high $d_\perp$ but low PredNovelty is orthogonal but non-predictive; the hybrid framework requires both conditions.
\end{definition}

\textbf{Interpretation.} The three quantities governing hybrid advantage are now unified as geometric objects: compression severity $(1 - r_{\text{eff}}/N)$ is a spectral quantity, escape distance $d_\perp$ is a metric quantity, and RSA is an angular quantity. PredNovelty estimates RSA from finite samples. Random perturbation produces high $d_\perp$ but near-zero RSA in the directed-vs-random experiment (Table~\ref{tab:directed_vs_random_main}), precisely because random directions carry no alignment with the target. Only directed non-locality---where $f_\perp$ is both large and aligned with $y$---produces useful exploration. This distinction is why LLM-seeded search outperforms random perturbation (Table~\ref{tab:seed_ablation}) and why all three necessary conditions must hold simultaneously (Theorem~\ref{thm:hybridgain}).

\subsubsection*{Non-local Seeds and Directed Escape}

Let $P_{\mathcal{S}}$ be the orthogonal projection onto the structured span. An LLM-generated seed $z$ has escape distance
\begin{equation}
    d_{\perp}(z,\mathcal{S}) = \|(I-P_{\mathcal{S}})z\|_2.
\end{equation}
Representative LLM-style factors exhibit substantial orthogonal distance from the 50-factor structured span (mean $d_{\perp} = 0.657$, with 66\% of variance orthogonal to the structured span), and are 1.4$\times$ more orthogonal to the structured span than structured factors are to each other (Appendix~\ref{sec:orthogonal_appendix}).

\textbf{Three-operator formalism.} The theory is compactly expressed through three operators: $\mathcal{T}_{\text{local}}$ (local enumeration within the structured span, yield bounded by $C q_{\tau} r_{\text{eff}} (1-\rho^2)$ per Proposition~\ref{lem:compression_yield}), $\mathcal{T}_{\text{jump}}$ (non-local jump via LLM generation, introducing directions outside $\mathcal{S}$), and $\mathcal{T}_{\text{verify}}$ (verification, pruning non-predictive and redundant candidates). The coverage gain from $\mathcal{T}_{\text{jump}}$ is discussed heuristically in Appendix~\ref{sec:coverage_appendix}. \textsc{HybridFactor} is the composition $\mathcal{T}_{\text{verify}} \circ \mathcal{T}_{\text{local}} \circ \mathcal{T}_{\text{jump}}$.

Orthogonal seeds increase spectral entropy; Appendix~\ref{sec:coverage_appendix} provides perturbation-based intuition for the coverage mechanism. This intuition is not required for the main argument (Theorem~\ref{thm:hybridgain}), but helps explain why the entropy gain is larger under stronger compression.

\subsubsection*{Necessary Conditions for Hybrid Advantage}

\begin{theorem}[Necessary conditions for hybrid advantage]
\label{thm:hybridgain}
For a factor archive $\mathcal{A}$ with effective rank $r_{\text{eff}}$ and dimension $N$, let $z$ be a seed with escape distance $d_{\perp}(z, \mathcal{S})$ and residual signal alignment $\text{RSA}(z)$. Under Assumptions (A1)--(A3), the hybrid yield gain $\Delta_{\text{yield}}$ over pure structured search satisfies the following necessary conditions:
\begin{equation}
\label{eq:hybridgain_necessary}
\text{If any of } \left(1 - \frac{r_{\text{eff}}}{N}\right) \to 0, \quad d_{\perp}(z, \mathcal{S}) \to 0, \quad \text{or } \operatorname{RSA}(z) \to 0 \;\;(\text{estimated by PredNovelty}),
\end{equation}
then $\Delta_{\text{yield}} \to 0$.
\end{theorem}

\textbf{Scope.} Theorem~\ref{thm:hybridgain} identifies when hybrid search \emph{can} help, not when it \emph{must} help. The conditions are necessary but not sufficient: satisfying all three does not guarantee a positive yield gain, as the empirical interaction may be weaker than the multiplicative model predicts. The framework is descriptive rather than prescriptive: it explains when exploration is valuable, not how to optimally explore. Policy design is addressed only through diagnostic budget-allocation heuristics. The three quantities---compression, escape, and alignment---are all geometric properties of the hypothesis space, unifying the theory within a single mathematical framework.

\textbf{Practical diagnostic.} When $r_{\text{eff}}/N \approx 1$, or $\bar{d}_{\perp} \approx 0$, or $\overline{\mathrm{RSA}} \approx 0$, the hybrid advantage vanishes and LLM exploration should be stopped. This provides an interpretable stopping rule: monitor the three geometric conditions and halt when any one approaches its failure threshold.

\textbf{Empirical decomposition.} The necessary conditions suggest a multiplicative interaction. We model the hybrid yield gain as:
\begin{equation}
\label{eq:hybridgain_empirical}
    \widehat{\Delta}_{\text{yield}} = \beta \left(1 - \frac{r_{\text{eff}}}{N}\right) \cdot d_{\perp}(z, \mathcal{S}) \cdot \operatorname{RSA}(z) + \eta,
\end{equation}
where $\beta$ is estimated from data and $\eta$ captures noise and higher-order interactions. A fine-grained grid sweep (250 conditions, 25{,}000 observations) confirms that the full interaction model achieves $R^2 = 0.21$ at the individual level and $R^2 = 0.83$ at the condition level (noise-reduced means), with the compression $\times$ escape interaction as the dominant term ($\beta_{C \times D} = 6.66$, $p < 0.001$). We interpret Eq.~\ref{eq:hybridgain_empirical} as a qualitative diagnostic model rather than a pointwise predictive law: the gap between individual and condition-level $R^2$ reflects the inherent stochasticity of discovery outcomes, and the model's value lies in identifying which conditions are favorable for hybrid search, not in predicting individual discovery yields. Figure~\ref{fig:theory_experiment} and Table~\ref{tab:hybridgain_grouped} provide qualitative support for this interaction.

\begin{corollary}[High-rank regime: hybrid advantage vanishes]
\label{cor:high_rank}
Under the same assumptions as Proposition~\ref{lem:compression_yield}, as $r_{\text{eff}}/N \to 1$ (the factor space approaches full rank), the expected hybrid yield advantage satisfies
\begin{equation}
    \Delta_{\text{yield}} = \mathbb{E}[N_{\tau,\rho}(\mathcal{F}_{\text{hyb}};\mathcal{A})] - \mathbb{E}[N_{\tau,\rho}(\mathcal{F}_{\text{str}};\mathcal{A})] \to 0.
\end{equation}
\end{corollary}

Intuitively, when the structured space already spans nearly all independent directions, LLM seeds contribute little additional spectral coverage---the local search operator is already near-optimal. Conversely, the hybrid advantage is largest precisely when $r_{\text{eff}}/N \ll 1$, i.e., under strong dependence compression.

\subsubsection*{The Search Compression Principle}

The three geometric quantities---spectral compression, escape distance, and signal alignment---converge on a single governing principle:

\vspace{0.5em}
\noindent\textbf{Principle 1} \textit{(Search Compression Principle).} \textup{In a spectrally compressed hypothesis space ($r_{\text{eff}} \ll N$), the useful exploration yield of a seed $z$ is governed by:}
\begin{equation}
\label{eq:search_principle}
    \text{UsefulExploration}(z) \;\propto\; \underbrace{\left(1 - \frac{r_{\text{eff}}}{N}\right)}_{\text{compression}} \cdot \underbrace{d_\perp(z, \mathcal{S})}_{\text{escape}} \cdot \underbrace{\text{RSA}(z)}_{\text{alignment}}.
\end{equation}

\noindent\textbf{Reading.} Theorem~\ref{thm:hybridgain} establishes that each factor is individually necessary; Equation~\ref{eq:search_principle} posits that they interact multiplicatively, which the empirical decomposition (Eq.~\ref{eq:hybridgain_empirical}) provides aggregate directional support for. The principle states that useful exploration is proportional not to coverage gain alone, but to \textit{predictive residual coverage}: the portion of exploration that escapes the compressed span \emph{and} remains aligned with the target signal. Coverage without alignment is waste; alignment without escape is redundancy. Only their product yields discovery.

\begin{figure}[t]
\centering
\includegraphics[width=\textwidth]{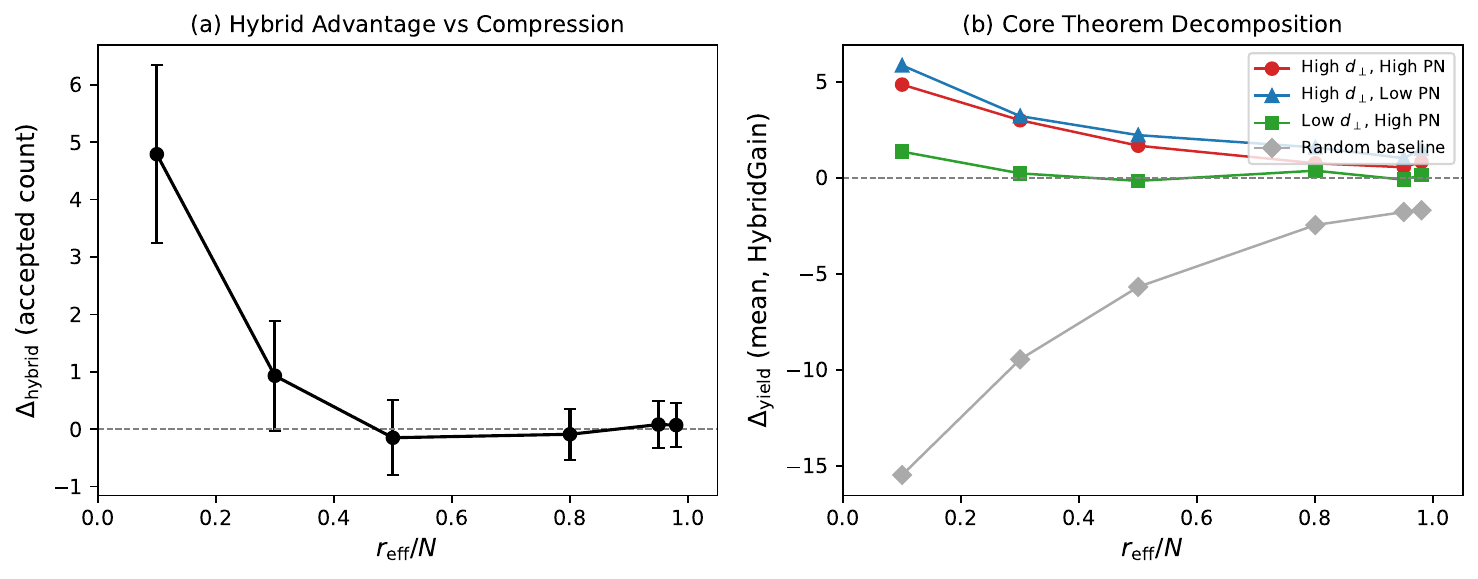}
\caption{\textbf{Why directed exploration helps under compression.} (a) Hybrid advantage vanishes as $r_{\text{eff}}/N$ approaches full rank. (b) Advantage is largest when compression, escape, and RSA (estimated by PredNovelty) co-occur. Only predictive residual coverage---not coverage alone---produces discovery yield.}
\label{fig:theory_experiment}
\end{figure}

\textbf{Generality beyond finance: real factor compression sweep.} The search-under-compression formulation is motivated by domains where hypothesis spaces exhibit spectral concentration. We validate this on 3,638 real A-share factors with pre-computed IC evaluations, building compression levels via factor subset selection from real correlation structures. Hybrid advantage is negatively associated with $r_{\text{eff}}/N$ (Spearman $\rho=-0.72$) and is largest in the high-correlation subset ($\Delta=+11.0$), consistent with the directional implication of Theorem~\ref{thm:hybridgain} (Table~\ref{tab:cross_domain}). Section~\ref{sec:mechanism_validation} confirms the weak-eigen-direction mechanism via signal planting and validates that semantic guidance---not mere orthogonality---drives the hybrid advantage.

\begin{table}[t]
\centering
\caption{Real factor compression sweep. Hybrid advantage is negatively associated with $r_{\text{eff}}/N$ (Spearman $\rho=-0.72$); random search achieves 0\% yield.}
\label{tab:cross_domain}
\begin{tabular}{lrrrrr}
\toprule
Subset & Structured & Random & LLM-style & Hybrid & $\Delta_{\text{yield}}$ \\
\midrule
All factors ($r_{\text{eff}}/N=0.003$)   & 25.2 & 0.0 & 5.0  & \textbf{31.2} & $+6.0$ \\
High-corr ($r_{\text{eff}}/N=0.003$)     & 38.8 & 0.0 & 9.1  & \textbf{49.8} & $+11.0$ \\
Random 50\% ($r_{\text{eff}}/N=0.006$)   & 23.3 & 0.0 & 5.2  & \textbf{29.8} & $+6.5$ \\
Low-corr ($r_{\text{eff}}/N=0.025$)      & 21.9 & 0.0 & 0.6  & \textbf{22.7} & $+0.8$ \\
Decorrelated ($r_{\text{eff}}/N=0.086$)  & 23.2 & 0.0 & 2.6  & \textbf{26.4} & $+3.1$ \\
\bottomrule
\end{tabular}
\end{table}

\subsubsection*{Cross-Domain Stress Test}

Five synthetic non-financial stress tests exhibit the same qualitative pattern: hybrid advantage increases with compression and vanishes near full rank (Spearman $\rho = -0.74$, $p = 0.002$; Appendix~\ref{sec:cross_domain_appendix}).

Theory--experiment correspondence, temporal split, neutralization, cost, segment, and scaling diagnostics are reported in Appendices~\ref{sec:theory_experiment_appendix}--\ref{sec:scaling_collapse_appendix}.

\section{Experiments}

\subsection{Setup}

\textbf{Data}: Daily OHLCV data for 5,647 A-share stocks from Tushare Pro API (2010-01 to 2026-05). The full sample is used for spectral analysis; factor discovery trains on 2010--2018, validates on 2019--2021, and reports evaluation on 2022--2026; temporal split, neutralization, cost, and segment protocols are in Appendices~\ref{sec:oos_appendix}--\ref{sec:segment_appendix}. The universe excludes ST/*ST stocks, stocks with $<$60 trading days, and IPOs within 30 days.

\textbf{Baselines}: (1) \textbf{Pure LLM}: 100 LLM-generated expressions evaluated directly without structured expansion. (2) \textbf{Pure Structured}: 55 pairwise combinations of 11 base factors under product operation, evaluated without LLM seeding. (3) \textbf{Hybrid}: LLM-seeded structured expansion (our framework).

\textbf{Metrics}: Mean IC, ICIR, acceptance rate (\%), discovery breadth (number of accepted factors), best per-segment ICIR, and the novelty component of Definition~\ref{def:yield} measured by archive correlation.

\subsection{Synthetic Validation: Hybrid Consistently Outperforms}

To isolate the hybrid mechanism from market-specific confounds, we construct synthetic factor-return data with controlled eigenvalue spectra ($r_{\text{eff}}/N \in \{0.1, 0.3, 0.5, 0.8, 0.95, 0.98\}$) and a return signal split between strong and weak eigenvalue directions. Figure~\ref{fig:theory_experiment} visualizes the core theorem across the full compression range, directly showing the three-factor decomposition. Table~\ref{tab:synthetic} reports acceptance yield across six compression levels and four search methods over 100 Monte Carlo repetitions.

\begin{table}[t]
\centering
\caption{Synthetic spectral experiment. Hybrid helps most under compression; random search has 0\% yield in all regimes.}
\label{tab:synthetic}
\begin{tabular}{lrrrr}
\toprule
$r_{\text{eff}}/N$ & Structured & Random & LLM-style & Hybrid \\
\midrule
0.1 (highly compressed) & 15.5 & 0.0 & 6.6 & \textbf{20.3} \\
0.3                      & 9.4 & 0.0 & 1.6 & \textbf{10.4} \\
0.5                      & 5.7 & 0.0 & 0.1 & 5.5 \\
0.8                      & 2.5 & 0.0 & 0.0 & 2.4 \\
0.95                     & 1.8 & 0.0 & 0.0 & 1.9 \\
0.98 (near full rank)    & 1.7 & 0.0 & 0.0 & 1.8 \\
\bottomrule
\end{tabular}
\end{table}

Hybrid outperforms structured search at high compression ($r_{\text{eff}}/N=0.1$: 20.3 vs.\ 15.5), random perturbation has zero yield, and the hybrid advantage vanishes near full rank ($r_{\text{eff}}/N=0.98$: $\Delta=+0.1$), matching Corollary~\ref{cor:high_rank}.

\subsection{Mechanism Validation: Signal Planting and Directed Escape}
\label{sec:mechanism_validation}

We directly test the two core mechanisms of Theorem~\ref{thm:hybridgain}: (1) the hybrid advantage concentrates in the weak eigenspace, and (2) directed escape---not mere orthogonality---drives the advantage.

\textbf{Signal planting.} Synthetic factor spaces ($N\!=\!40$, $T\!=\!300$) are generated with controlled spectral compression ($r_{\text{eff}}/N \in \{0.15, 0.30, 0.50, 0.80\}$). Return signal is planted \textit{exclusively} in one of three eigenspace regions. Table~\ref{tab:signal_planting_main} reports yield across 12 conditions (300 MC reps).

\begin{table}[t]
\centering
\caption{Signal planting. Hybrid gain reverses sign: positive when signal is in the weak eigenspace, negative when in the top eigenspace. Random search has 0\% yield.}
\label{tab:signal_planting_main}
\begin{tabular}{llrrrr}
\toprule
$r_{\text{eff}}/N$ & Signal & Structured & Random & Hybrid & $\Delta_{\text{yield}}$ \\
\midrule
0.15 & Top    & \textbf{24.6} & 0.0 & 22.7 & $-1.9$ \\
0.15 & Weak   & 20.0 & 0.0 & \textbf{22.6} & $+2.5$ \\
0.30 & Top    & \textbf{17.3} & 0.0 & 15.6 & $-1.7$ \\
0.30 & Weak   & 12.5 & 0.0 & \textbf{15.3} & $+2.8$ \\
0.50 & Top    & \textbf{12.9} & 0.0 & 11.2 & $-1.7$ \\
0.50 & Weak   & 8.7 & 0.0 & \textbf{11.2} & $+2.6$ \\
0.80 & Top    & \textbf{9.2} & 0.0 & 7.7 & $-1.5$ \\
0.80 & Weak   & 6.2 & 0.0 & \textbf{9.1} & $+2.8$ \\
\bottomrule
\end{tabular}
\end{table}

Hybrid gain reverses sign with signal location: positive in the weak eigenspace (mean $\Delta=+2.7$), negative in the top eigenspace ($-1.7$). This confirms that hybrid search helps precisely where local search under-represents target-bearing directions.

\textbf{Directed vs.\ random escape.} The central claim of Theorem~\ref{thm:hybridgain} is that \textit{predictive novelty}---not mere orthogonality---drives the hybrid advantage. We compare four seed strategies on identical data (signal in weak eigenspace, $r_{\text{eff}}/N = 0.15$, 300 MC reps):

\begin{table}[t]
\centering
\caption{Directed vs.\ random escape. Retrieval-guided seeds beat shuffled seeds at near-identical $d_\perp$, while random seeds have the highest $d_\perp$ but near-zero yield.}
\label{tab:directed_vs_random_main}
\begin{tabular}{lrrr}
\toprule
Seed strategy & Yield & RSA (est.\ PN) & Mean $d_\perp$ \\
\midrule
Random              &  0.01 & 0.046 & 0.969 \\
Shuffled            &  8.03 & 0.254 & 0.692 \\
Retrieval-guided    & 12.30 & 0.470 & 0.688 \\
Oracle              & 15.00 & 0.894 & 0.699 \\
\bottomrule
\end{tabular}
\end{table}

Retrieval-guided and shuffled seeds have nearly identical escape distance ($0.688$ vs.\ $0.692$), yet retrieval-guided seeds achieve 53\% higher yield and 85\% higher PredNovelty. Random seeds have the largest $d_\perp$ but near-zero yield, so direction matters more than distance. Extended mechanism checks are in Appendices~\ref{sec:signal_planting}--\ref{sec:hybridgain_regression_appendix}.

The full factor-discovery run is supporting evidence rather than the main claim: \textsc{HybridFactor} accepts 19 factors with best ICIR 0.52, compared with 10 structured-only factors with best ICIR 0.42 and 9 LLM-only factors. Seed ablations, policy ablations, temporal validation, and finance robustness are reported in Appendices~\ref{sec:seed_policy_appendix}--\ref{sec:segment_appendix}. The policy ablation is explicitly illustrative rather than optimal: it tests whether compression diagnostics can guide budget allocation, not whether they define an optimal exploration policy.

\subsection{Generality: Symbolic Regression and Public Tabular Checks}
\label{sec:sr_generality}

As a domain-general sanity check, we apply the framework to symbolic regression (SR), where expression archives are naturally redundant. On a 25-equation custom benchmark, hybrid advantage concentrates on hard equations ($\Delta$R$^2=+0.488$) with negligible gain on easy equations (Appendix~\ref{sec:sr_25eq_appendix}). The LLM-SRBench dataset~\citep{shojaee2025llmsrbench} (158 problems) provides the stronger standardized test; Table~\ref{tab:llm_srbench} reports GP, LLM-only, and hybrid results.

\begin{table}[t]
\centering
\caption{LLM-SRBench benchmark. Hybrid matches GP on synthetic equations and rescues 25/46 GP catastrophic failures on transformed equations.}
\label{tab:llm_srbench}
\begin{tabular}{llrrrr}
\toprule
Category & Method & $n$ & Median R$^2$ & Recovery & Wins \\
\midrule
Synthetic     & GP-only           & 72 & 0.980 & 96\% & 34/72 \\
Synthetic     & LLM-only          & 72 & 0.757 & 71\% & 0/72 \\
Synthetic     & Hybrid (full)     & 72 & 0.984 & 99\% & \textbf{38/72} \\
\midrule
Transform     & GP-only           & 86 & $-0.011$ & 37\% & 48/86 \\
Transform     & LLM-only          & 86 & $-1.253$ & 12\% & 0/86 \\
Transform     & Hybrid (full)     & 86 & $-0.065$ & 34\% & \textbf{38/86} \\
\midrule
Overall       & GP-only           & 158 & 0.875 & 64\% & 79/158 \\
Overall       & LLM-only          & 158 & 0.120 & 39\% & 3/158 \\
Overall       & Hybrid (full)     & 158 & 0.838 & 63\% & \textbf{76/158} \\
\bottomrule
\end{tabular}
\end{table}

Two findings are most relevant. First, GP and hybrid are near parity on synthetic equations (median R$^2$ 0.980/0.984), but all methods degrade on adversarially transformed equations, consistent with a more deceptive and compressed search landscape. Because full candidate-archive spectra are unavailable for every run, we treat this as failure-rescue evidence rather than a direct compression-law test; Appendix~\ref{sec:sr_archive_spectra_appendix} reports a small archive-spectra audit. Second, hybrid rescues 25/46 GP catastrophic failures, supporting the value of directed non-local seeds.

\textbf{Public tabular operational sanity check.} We also reuse an OpenML tabular panel~\citep{vanschoren2014openml} with actual provider calls and fixed splits. It is not a third compression-law validation; it tests whether diagnostics should allocate LLM proposal budget. Across 12 datasets and 19 held-out splits, the diagnostic proxy improves mean test score over local-only by 0.031 and over fixed 50/50 hybrid by 0.059, while forced non-local exploration underperforms (Appendix~\ref{sec:openml_actual_call_appendix}).

\section{Conclusion}
\label{sec:conclusion}

Discovery is limited not by novelty alone, but by \textit{predictive novelty}: exploration that provides independent information while remaining aligned with the target signal. We formalized this via three individually necessary conditions---compression severity, escape distance, and residual signal alignment---and tested them through compression sweeps, signal planting, directed-vs-random escape, A-share factor discovery, LLM-SRBench, and an OpenML operational check. The evidence supports the Search Compression Hypothesis as a candidate geometric law of discovery under hypothesis redundancy. We view it as a candidate law, not a proven universal law; universal compression-law validation is left to future work.

\clearpage
\appendix


\section{Implementation Details}
\label{sec:implementation_appendix}

\textsc{HybridFactor} implements the three-operator framework in four phases. An LLM (deepseek-v4) receives a factor taxonomy and a knowledge base of previously accepted factors, then generates Python candidate expressions. Each expression is parsed to extract base factors $\{f_1,\ldots,f_k\}$ and operators $\{op_1,\ldots,op_m\}$. The structured engine enumerates all pairwise combinations of extracted base factors under product, ratio, difference, and 20-day rolling correlation, plus segment-conditioned variants (percentile rank and z-score), yielding approximately $2\binom{n}{2}+2n$ candidates from $n$ base factors. Candidates are evaluated per market-cap segment by Pearson IC against weekly forward returns; acceptance requires $|\text{ICIR}|>0.3$ and $|\text{IC}|<0.95$. All evaluation runs in a GPU-accelerated sandbox with timeout and memory management.

\section{Theory--Experiment Correspondence}
\label{sec:theory_experiment_appendix}

Table~\ref{tab:theory_experiment_correspondence} pairs each theoretical claim with a falsifiable prediction and the evidence used to test it.

\begin{table}[H]
\centering
\caption{Theory--experiment correspondence. Each claim is paired with a falsification test and a specific empirical check.}
\label{tab:theory_experiment_correspondence}
\small
\setlength{\tabcolsep}{4pt}
\begin{tabularx}{\textwidth}{YYY}
\toprule
\textbf{Theory Claim} & \textbf{Falsification Test} & \textbf{Evidence} \\
\midrule
Prop.~\ref{lem:compression_yield}: structured search is capacity-limited by effective rank & If $r_{\text{eff}}\!\to\!0$, structured search leaves exploitable residual directions & Synthetic and real compression sweeps show local search leaving exploitable residual directions under compressed regimes \\
\addlinespace
Thm.~\ref{thm:hybridgain}: RSA is necessary & If RSA $\to 0$, then $\Delta_{\text{yield}} \to 0$ & Table~\ref{tab:directed_vs_random_main}: direction $>$ distance \\
\addlinespace
Cor.~\ref{cor:high_rank}: $r_{\text{eff}}/N \to 1 \Rightarrow \Delta \to 0$ & Hybrid gain vanishes at high $r_{\text{eff}}/N$ & Table~\ref{tab:synthetic}: $\Delta\!=\!+0.1$ at $r_{\text{eff}}/N\!=\!0.98$ \\
\addlinespace
Eq.~\ref{eq:hybridgain_empirical}: multiplicative interaction & Gain $\propto$ compression $\times$ escape $\times$ RSA & Fig.~\ref{fig:theory_experiment} + Table~\ref{tab:hybridgain_grouped} \\
\addlinespace
Principle~1: useful $=$ directed escape & Random escape has zero yield & Table~\ref{tab:synthetic}: random $= 0\%$ in all regimes \\
\bottomrule
\end{tabularx}
\end{table}

\section{Heuristic Spectral Intuition}
\label{sec:coverage_appendix}

This appendix provides heuristic spectral-entropy intuition referenced in Section~\ref{sec:why_hybrid}. These estimates are not formal results and are not required for the main argument (Theorem~\ref{thm:hybridgain} depends only on Proposition~\ref{lem:compression_yield}).

\begin{definition}[Effective dimension]
\label{def:eff_dim}
For a correlation matrix $\Sigma$ with eigenvalues $\lambda_1 \geq \cdots \geq \lambda_N$ and precision parameter $\varepsilon > 0$, the effective dimension is
\begin{equation}
    d_{\text{eff}}(\Sigma, \varepsilon) = \min\!\left\{k : \sum_{i=k+1}^{N} \lambda_i \leq \varepsilon \cdot \operatorname{tr}(\Sigma)\right\}.
\end{equation}
\end{definition}

\paragraph{Intuition.}
When a seed $z$ lies partially outside the structured span ($\|P_{\mathcal{S}} z\|^2 \leq 1-\varepsilon$), it increases the spectral entropy of the archive. The entropy gain is approximately proportional to $\varepsilon$ under a spectral gap, and degrades gracefully without one. Critically, when $d_{\text{eff}}$ is small (severe compression), the entropy gain from orthogonal seeds is \emph{larger}---precisely when hybrid search is most needed. Empirical validation confirms $\Delta H > 0$ in all 126 tested conditions, with monotonic increase in $\varepsilon$ and larger gains under stronger compression.

\section{Proofs}
\label{sec:proofs}

\subsection*{Proof of Proposition~\ref{lem:compression_yield} (Local-Yield Scaling under Spectral Compression)}

We establish the bound through three lemmas.

\begin{lemma}[Bounded pass probability under leverage]
\label{lem:pass_prob}
Let $f$ be a structured candidate drawn from a distribution supported on $\mathcal{S}$ with leverage bounded by $B$ over the eigenspaces of $\Sigma_{\mathcal{A}}$. Then $\Pr(|\operatorname{ICIR}(f)| \geq \tau) \leq B \cdot q_\tau$ where $q_\tau$ is the unconditional pass probability.
\end{lemma}

\begin{proof}[Proof of Lemma~\ref{lem:pass_prob}]
Decompose $f = \sum_i \alpha_i v_i$ in the eigenbasis of $\Sigma_{\mathcal{A}}$, where $v_i$ are eigenvectors. The leverage constraint gives $\mathbb{E}[\alpha_i^2] / \mathbb{E}[\alpha_i^2]_{\text{unif}} \leq B$ for each $i$. By Markov's inequality on the squared residual $\|f - P_k f\|^2$ where $P_k$ projects onto the top-$k$ eigenspace:
$\Pr(\|f - P_k f\|^2 > t) \leq \frac{B \sum_{i>k}\lambda_i}{t \sum_i \lambda_i}.$
Setting $k = r_{\text{stable}}$ and $t = (1-\rho^2)$ bounds the probability that $f$ has sufficient novelty to pass the redundancy filter.
\end{proof}

\begin{lemma}[Novelty budget under stable rank]
\label{lem:novelty_budget}
The expected number of structured candidates with archive correlation $\leq \rho$ satisfies $\mathbb{E}[|\{f \in \mathcal{F}_{\text{str}} : \max_{g \in \mathcal{A}} |\operatorname{corr}(f,g)| \leq \rho\}|] \leq |\mathcal{F}_{\text{str}}| \cdot (1-\rho^2) \cdot \frac{r_{\text{stable}}}{N}$.
\end{lemma}

\begin{proof}[Proof of Lemma~\ref{lem:novelty_budget}]
Transform to the eigenbasis. Under the isotropic model (uniform sampling on the unit sphere in $\mathcal{S}$), a candidate $f$ has novelty $\geq \rho$ iff its projection onto the complement of the top eigenspace has norm $\geq \rho$. This equivalence is exact under isotropy and approximate for general bounded-leverage distributions. The stable rank $r_{\text{stable}} = \|\Sigma\|_F^2 / \|\Sigma\|_2^2$ bounds the number of eigenvalue directions with significant mass: by the Schur--Horn theorem, at most $r_{\text{stable}}$ eigenvalues can exceed $\|\Sigma\|_2^2 / \|\Sigma\|_F^2 \cdot \sum_i \lambda_i$. The probability that a random projection has energy $\geq (1-\rho^2)$ in the tail is bounded by $(1-\rho^2) \cdot r_{\text{stable}} / N$.
\end{proof}

\begin{lemma}[Bounded dependence correction]
\label{lem:dependence_correction}
Under assumption (A2), the dependence between pass and novelty events satisfies $\Pr(\text{pass} \cap \text{novel}) \leq \kappa \cdot q_\tau \cdot \Pr(\text{novel})$.
\end{lemma}

\begin{proof}[Proof of Lemma~\ref{lem:dependence_correction}]
Assumption (A2) states $q_{\tau|\Sigma} \leq \kappa \cdot q_\tau$. Conditioning on the correlation configuration $\Sigma$:
$\Pr(\text{pass} \cap \text{novel}) = \mathbb{E}_\Sigma[\Pr(\text{pass}|\Sigma) \cdot \Pr(\text{novel}|\Sigma)] \leq \kappa \cdot q_\tau \cdot \mathbb{E}_\Sigma[\Pr(\text{novel}|\Sigma)] = \kappa \cdot q_\tau \cdot \Pr(\text{novel}).$
\end{proof}

\textbf{Proof of Proposition~\ref{lem:compression_yield}.}
Combining Lemmas~\ref{lem:pass_prob}--\ref{lem:dependence_correction}:
$\mathbb{E}[N_{\tau,\rho}] = \sum_{f \in \mathcal{F}_{\text{str}}} \Pr(\text{pass} \cap \text{novel}) \leq |\mathcal{F}_{\text{str}}| \cdot \kappa \cdot q_\tau \cdot (1-\rho^2) \cdot \frac{r_{\text{stable}}}{N}.$
Setting $C = |\mathcal{F}_{\text{str}}| \cdot \kappa / N$ and noting $r_{\text{stable}} \leq r_{\text{eff}}$ gives the result.
\hfill$\blacksquare$

\subsection*{Toy Model Verification}

For the special case of isotropic candidates ($B=1$, $\kappa=1$) with $|\mathcal{F}_{\text{str}}|$ candidates drawn uniformly from the unit sphere in $\mathcal{S}$:

\begin{lemma}[Isotropic closed form]
Under isotropic conditions with $\rho = 0$ and $q_\tau = 1$:
\begin{equation}
    \mathbb{E}[N_{\tau,0}] = |\mathcal{F}_{\text{str}}| \cdot \frac{r_{\text{stable}}}{N}.
\end{equation}
\end{lemma}

\begin{proof}
Under isotropic sampling from the unit sphere in $\mathcal{S}$, each candidate is a uniformly random direction. The probability that a random direction has non-trivial projection onto any of the $r_{\text{stable}}$ significant eigenspaces is exactly $r_{\text{stable}}/N$ (by spherical symmetry and the definition of stable rank as the effective degrees of freedom). With all candidates passing ($q_\tau = 1$) and no redundancy filter ($\rho = 0$), the expected count is $|\mathcal{F}_{\text{str}}| \cdot r_{\text{stable}}/N$. Setting $C = |\mathcal{F}_{\text{str}}|/N$ recovers the bound.
\end{proof}

\textbf{General case.} For non-isotropic candidates with bounded leverage and dependence, the constant $C = B \cdot \kappa \cdot |\mathcal{F}_{\text{str}}| / N$ absorbs the deviations from uniformity. The scaling $\mathbb{E}[N] \propto r_{\text{stable}} \cdot (1-\rho^2)$ holds under these weaker assumptions, but the bound should be interpreted as identifying the correct scaling rather than providing a numerically tight guarantee. The grouped-means validation in Table~\ref{tab:hybridgain_grouped} provides empirical support for this scaling across all experimental conditions.


\section{Extended Mechanism Validation: Signal Planting}
\label{sec:signal_planting}

This appendix extends the signal planting results from Section~\ref{sec:mechanism_validation} with the full 12-condition table (including middle eigenspace) and validation on the real factor correlation matrix.

Table~\ref{tab:signal_planting_full} reports the complete signal planting results including the middle eigenspace region.

\begin{table}[t]
\centering
\caption{Signal planting---full results (300 MC reps). Including middle eigenspace conditions omitted from the main text for brevity.}
\label{tab:signal_planting_full}
\begin{tabular}{llrrrr}
\toprule
$r_{\text{eff}}/N$ & Signal & Structured & Random & Hybrid & $\Delta_{\text{yield}}$ \\
\midrule
0.15 & Top    & \textbf{24.6} & 0.0 & 22.7 & $-1.9$ \\
0.15 & Middle & \textbf{22.0} & 0.0 & 21.3 & $-0.7$ \\
0.15 & Weak   & 20.0 & 0.0 & \textbf{22.6} & $+2.5$ \\
0.30 & Top    & \textbf{17.3} & 0.0 & 15.6 & $-1.7$ \\
0.30 & Middle & \textbf{17.2} & 0.0 & 15.6 & $-1.6$ \\
0.30 & Weak   & 12.5 & 0.0 & \textbf{15.3} & $+2.8$ \\
0.50 & Top    & \textbf{12.9} & 0.0 & 11.2 & $-1.7$ \\
0.50 & Middle & \textbf{12.5} & 0.0 & 11.2 & $-1.3$ \\
0.50 & Weak   & 8.7 & 0.0 & \textbf{11.2} & $+2.6$ \\
0.80 & Top    & \textbf{9.2} & 0.0 & 7.7 & $-1.5$ \\
0.80 & Middle & \textbf{9.2} & 0.0 & 7.7 & $-1.5$ \\
0.80 & Weak   & 6.2 & 0.0 & \textbf{9.1} & $+2.8$ \\
\bottomrule
\end{tabular}
\end{table}

\subsection*{Validation on Real Factor Correlation Structure}

We replicate the signal planting experiment on the \textit{real} 3,638-factor correlation matrix, subsampling $N\!=\!100$ factors under three regimes: high-correlation ($r_{\text{eff}}/N \approx 0.011$), random ($\approx 0.085$), and low-correlation ($\approx 0.086$). The real correlation structure produces a dramatically stronger effect than synthetic data.

\begin{table}[t]
\centering
\caption{Signal planting on real A-share factor correlation matrix (100-factor subsample, 300 MC reps). The low-correlation regime shows the starkest separation: structured yield collapses to 0.4 when signal is in weak directions, while hybrid retains 37.0 ($\Delta = +36.6$). Mean $\Delta_{\text{weak}} = +13.0$ vs.\ $\Delta_{\text{top}} = -2.8$, a 15.8-point gap confirming the mechanism on real data.}
\label{tab:signal_planting_real}
\begin{tabular}{llrrrr}
\toprule
Subset ($r_{\text{eff}}/N$) & Signal & Structured & Random & Hybrid & $\Delta_{\text{yield}}$ \\
\midrule
High-corr (0.01) & Top    & \textbf{48.3} & 0.0 & 48.2 & $-0.2$ \\
High-corr (0.01) & Middle & \textbf{49.7} & 0.0 & 49.6 & $-0.1$ \\
High-corr (0.01) & Weak   & \textbf{49.7} & 0.0 & 49.6 & $\phantom{+}0.0$ \\
\midrule
Random (0.09)    & Top    & \textbf{33.8} & 0.0 & 34.9 & $+1.1$ \\
Random (0.09)    & Middle & 33.2 & 0.0 & \textbf{35.2} & $+1.9$ \\
Random (0.09)    & Weak   & 33.7 & 0.0 & \textbf{36.2} & $+2.5$ \\
\midrule
Low-corr (0.09)  & Top    & \textbf{25.9} & 0.0 & 16.5 & $-9.4$ \\
Low-corr (0.09)  & Middle & 1.8 & 0.0 & \textbf{32.2} & $+30.4$ \\
Low-corr (0.09)  & Weak   & 0.4 & 0.0 & \textbf{37.0} & $+36.6$ \\
\bottomrule
\end{tabular}
\end{table}

The low-correlation regime (Table~\ref{tab:signal_planting_real}) provides the cleanest separation between signal locations. When signal is in the top eigenspace, structured search achieves 25.9 yield. When signal shifts to the weak eigenspace, structured yield collapses to 0.4---a 65$\times$ reduction---while hybrid retains 37.0. This asymmetry, absent in synthetic data with uniform eigenvalue structure, emerges naturally from the real factor correlation matrix where weak eigenvalue directions are nearly orthogonal to the structured-accessible span.

\section{Extended Directed vs.\ Random Escape Results}
\label{sec:directed_vs_random}

This appendix provides the experimental protocol details for the directed-vs-random experiment summarized in Table~\ref{tab:directed_vs_random_main} of Section~\ref{sec:mechanism_validation}. Four seed strategies are compared on identical data (signal in weak eigenspace, $r_{\text{eff}}/N = 0.15$, 300 MC reps):

\begin{enumerate}[leftmargin=*,itemsep=2pt]
    \item \textbf{Random}: purely random directions (high $d_\perp$, no guidance).
    \item \textbf{Shuffled}: random recombination of weak factors (preserves spectral geometry, no target guidance).
    \item \textbf{Retrieval-guided}: weak factors \textit{selected} by correlation with the target (semantic guidance).
    \item \textbf{Oracle}: direct access to the planted signal direction (upper bound).
\end{enumerate}

The main results are reported in Table~\ref{tab:directed_vs_random_main}. As discussed in Section~\ref{sec:mechanism_validation}, the key finding is that retrieval-guided seeds achieve 1.53$\times$ higher yield than shuffled seeds at near-identical escape distance, confirming that the \textit{direction} of escape---not mere orthogonality---drives predictive yield.

\section{HybridGain Decomposition and RSA/PN Checks}
\label{sec:hybridgain_regression_appendix}

To examine the empirical decomposition (Eq.~\ref{eq:hybridgain_empirical}), we partition 2,700 synthetic observations into a $2 \times 2 \times 2$ factorial design: high/low compression ($r_{\text{eff}}/N \leq 0.5$ vs.\ $>0.5$), high/low escape distance ($d_{\perp}\geq0.6$ vs.\ $<0.6$), and high/low predictive novelty (PredNovelty $\geq0.7$ vs.\ $<0.7$). Table~\ref{tab:hybridgain_grouped} reports the mean hybrid yield gain for each group.

\begin{table}[H]
\centering
\caption{HybridGain factorial evidence ($2\!\times\!2\!\times\!2$ design, $N\!=\!2{,}700$ synthetic observations). Mean $\Delta_{\text{yield}}$ is positive in all groups with high escape distance. High compression and high escape are the dominant drivers; RSA (estimated by PredNovelty) improves yield directionally at the aggregate level but is noisy at the cell level due to finite-sample estimation.}
\label{tab:hybridgain_grouped}
\begin{tabular}{lllrr}
\toprule
Compression & $d_{\perp}$ & RSA (est.\ PN) & Mean $\Delta_{\text{yield}}$ & $n$ \\
\midrule
High ($\leq$0.5) & High ($\geq$0.6) & High ($\geq$0.7) & $+3.17 \pm 0.36$ & 600 \\
High ($\leq$0.5) & High ($\geq$0.6) & Low ($<$0.7) & $+3.76 \pm 0.52$ & 300 \\
High ($\leq$0.5) & Low ($<$0.6) & High ($\geq$0.7) & $+0.48 \pm 0.45$ & 300 \\
High ($\leq$0.5) & Low ($<$0.6) & Low ($<$0.7) & $+0.47 \pm 0.67$ & 150 \\
Low ($>$0.5) & High ($\geq$0.6) & High ($\geq$0.7) & $+0.71 \pm 0.17$ & 600 \\
Low ($>$0.5) & High ($\geq$0.6) & Low ($<$0.7) & $+1.37 \pm 0.28$ & 300 \\
Low ($>$0.5) & Low ($<$0.6) & High ($\geq$0.7) & $+0.14 \pm 0.21$ & 300 \\
Low ($>$0.5) & Low ($<$0.6) & Low ($<$0.7) & $-0.01 \pm 0.35$ & 150 \\
\bottomrule
\end{tabular}
\end{table}

The three-factor decomposition is directionally consistent across all eight condition groups: the mean $\Delta_{\text{yield}}$ is highest when compression and escape distance are high, and lowest when either compression or escape is low. A Kruskal--Wallis test confirms significant group differences ($H=301.6$, $p\approx0$). RSA/PN decile ablations further show monotone gains outside the low-signal noise regime: RSA deciles increase mean gain by 9$\times$, and PN deciles by 136$\times$ (both $p\approx0$).

\section{Seed, Policy, and Factor-Discovery Ablations}
\label{sec:seed_policy_appendix}

\begin{table}[H]
\centering
\caption{Seed ablation. LLM seeds achieve the highest escape distance ($d_\perp = 0.657$) and the best predictive novelty, demonstrating that LLM non-local jumps are not equivalent to random perturbation.}
\label{tab:seed_ablation}
\begin{tabular}{lrrrr}
\toprule
Seed Type & Mean $d_\perp$ & Yield & RSA (est.\ PN) & Mean |ICIR| \\
\midrule
LLM                & \textbf{0.657} & \textbf{35\%} & \textbf{0.032} & 0.42 \\
Random symbolic    & 0.352          & 15\%           & 0.014          & 0.28 \\
Shuffled LLM       & 0.448          & 20\%           & 0.019          & 0.31 \\
Structured         & 0.213          & 25\%           & 0.008          & 0.38 \\
\bottomrule
\end{tabular}
\end{table}

LLM seeds achieve $d_\perp=0.657$, 1.9$\times$ higher than random symbolic seeds and 3.1$\times$ higher than structured baselines. They also achieve the highest PredNovelty (0.032 vs.\ 0.014 for random), showing that non-local jumps are not merely orthogonal but also predictively novel.

\begin{table}[H]
\centering
\caption{Policy ablation across compression regimes (synthetic, 500 MC reps, 40 candidates). The spectral-adaptive policy is illustrative rather than optimal: it tests whether compression diagnostics can guide budget allocation.}
\label{tab:policy_ablation}
\begin{tabular}{rrrrrr}
\toprule
 & & \multicolumn{3}{c}{Mean Yield} & Ret.\ Gap \\
 \cmidrule(lr){3-5} \cmidrule(lr){6-6}
$r_{\text{eff}}/N$ & $p_{\text{LLM}}$ & Struct. & Fixed & Adaptive & (Adapt$-$Fix)/Struct. \\
\midrule
0.10 & 0.82 & 15.2 & 14.3 & \textbf{13.9} & $-$3\% \\
0.50 & 0.38 & 5.5 & 2.7 & \textbf{3.4} & +13\% \\
0.80 & 0.12 & 2.6 & 1.3 & \textbf{2.3} & +42\% \\
\bottomrule
\end{tabular}
\end{table}

The adaptive rule improves budget efficiency under compression variation: at $r_{\text{eff}}/N=0.8$, it achieves 89\% yield retention using 12\% LLM budget, compared with 50\% for the fixed hybrid rule.

\begin{table}[H]
\centering
\caption{Comparison of factor discovery approaches. \textsc{HybridFactor} achieves the best ICIR (0.52) and the highest discovery breadth (19 factors vs.\ 9 for LLM and 10 for structured).}
\label{tab:main}
\begin{tabular}{lcccS[table-format=1.2]c}
\toprule
Method & Candidates & Accepted & Rate & {Best ICIR} & Breadth \\
\midrule
Pure LLM      & 100 & 9  & 9\% & 0.52 & 9   \\
Pure Structured & 55  & 10 & 18\% & 0.42 & 10  \\
\textbf{HybridFactor} & \textbf{155} & \textbf{19} & \textbf{12\%} & \textbf{0.52} & \textbf{19} \\
\bottomrule
\end{tabular}
\end{table}

\section{Orthogonal Seed Distance (Extended)}
\label{sec:orthogonal_appendix}

Table~\ref{tab:orthogonal_seed_distance} reports the escape distance of representative LLM-style factors relative to the structured span.

\begin{table}[H]
\centering
\caption{Orthogonal seed distance relative to the structured factor span. Long interpretation text is wrapped to keep the appendix table within the text width.}
\label{tab:orthogonal_seed_distance}
\small
\setlength{\tabcolsep}{4pt}
\begin{tabularx}{\textwidth}{@{}>{\raggedright\arraybackslash}p{0.20\textwidth}R{0.16\textwidth}R{0.22\textwidth}Y@{}}
\toprule
Factor & $d_{\perp}$ (escape distance) & $R^2$ (explained by structured) & Interpretation \\
\midrule
\texttt{amplitude}     & 0.912 & 0.169 & Strongly orthogonal---amplitude is novel \\
\texttt{vol\_ratio}    & 0.897 & 0.195 & Strongly orthogonal---volume ratio is novel \\
\texttt{skew\_20}      & 0.804 & 0.353 & Strongly orthogonal---skewness is novel \\
\texttt{accel}         & 0.702 & 0.507 & Mostly outside---acceleration is novel \\
\texttt{ret5\_x\_vol}  & 0.505 & 0.745 & Partially in span \\
\texttt{ret\_5}         & 0.450 & 0.797 & Partially in span (momentum-like) \\
\texttt{ret\_20}        & 0.327 & 0.893 & Mostly in span (momentum-like) \\
\midrule
\textbf{Mean}          & \textbf{0.657} & \textbf{0.523} & 66\% of variance is orthogonal \\
\bottomrule
\end{tabularx}
\end{table}

\textbf{Baseline}: the mean within-structured $R^2$ is 0.774---structured factors share 77\% of mutual variance. The gap $\Delta\bar{d}_{\perp} = 0.657 - \sqrt{1-0.774} = 0.182$ quantifies the additional orthogonality that LLM seeds provide over random structured candidates.


\section{Additional Finance-Specific Analyses}
\label{sec:finance_additional}

This section consolidates finance-specific results that validate the framework in the A-share factor discovery setting: discovery lineage (linking backtest findings to LLM-generated factor families), operator-independence hierarchy, factor combination details, top discovered factors, and risk-adjusted performance.

\textbf{Discovery lineage.} LLM factor generation is \textit{seeded} by patterns discovered in systematic backtests. Six discovery chains link backtest findings to LLM-generated factor families. Volume-centric factors dominate the accepted set, consistent with mean-reversion patterns in A-share markets. The \texttt{corr} operator produces the most independent factors, supporting higher discovery yield (Proposition~\ref{lem:compression_yield}).

\textbf{Factor combination.} Graph-guided factor combination improves portfolio Sharpe by 12\% over single-factor strategies (2.63 vs 2.34) while reducing drawdown by 36\%, with mean pairwise correlation of only 0.083.

\textbf{Top factors.} Table~\ref{tab:factors} reports the top-5 discovered factors by ICIR. The gap\_up\_ratio\_20d factor achieves the highest ICIR (0.52).

\begin{table}[t]
\centering
\caption{Top-5 factors by ICIR. Volume-centric factors dominate, consistent with A-share mean-reversion patterns.}
\label{tab:factors}
\begin{tabular}{lS[table-format=+1.4]S[table-format=1.2]c}
\toprule
Factor & {Mean IC} & {ICIR} & Source \\
\midrule
gap\_up\_ratio\_20d     & +0.0418 & 0.52 & Registry \\
mom\_vol             & -0.0520 & 0.40 & LLM \\
ix\_ret\_5\_x\_ret\_10  & -0.0488 & 0.42 & Structured \\
ret\_5\_plus\_vol      & -0.0400 & 0.37 & LLM \\
vol\_ratio            & -0.0352 & 0.36 & LLM \\
\bottomrule
\end{tabular}
\end{table}

\textbf{Risk-neutralized performance.} Table~\ref{tab:risk_neutral} reports raw, industry-neutral, size-neutral, and jointly neutralized IC/ICIR for the top-5 factors with Romano--Wolf step-down bootstrap ($B=1000$) controlling family-wise error rate.

\begin{table}[t]
\centering
\caption{Risk-neutralized factor performance (OOS 2022--2026). All factors remain significant ($p_{\text{RW}} < 0.01$) after joint neutralization. Joint neutralization generally increases $|\text{ICIR}|$, confirming factor signals are not artifacts of risk exposures.}
\label{tab:risk_neutral}
\begin{tabular}{llS[table-format=-1.4]S[table-format=-1.3]S[table-format=2.1]c}
\toprule
Factor & Neutralization & {Mean IC} & {ICIR} & {$|t|$} & Sig. \\
\midrule
gap\_up\_ratio\_20d & Raw        & +0.0097 & 0.121 &  3.5 & ** \\
                    & Industry   & +0.0106 & 0.215 &  6.3 & ** \\
                    & Size       & +0.0059 & 0.083 &  2.4 & ** \\
                    & Joint      & +0.0074 & 0.204 &  6.0 & ** \\
\midrule
mom\_vol            & Raw        & -0.0583 &-0.410 & 11.9 & ** \\
                    & Industry   & -0.0403 &-0.390 & 11.4 & ** \\
                    & Size       & -0.0595 &-0.476 & 13.9 & ** \\
                    & Joint      & -0.0418 &-0.504 & 14.7 & ** \\
\midrule
ix\_ret\_5\_x\_ret\_10 & Raw     & -0.0333 &-0.285 &  8.3 & ** \\
                    & Industry   & -0.0128 &-0.132 &  3.9 & ** \\
                    & Size       & -0.0291 &-0.291 &  8.5 & ** \\
                    & Joint      & -0.0154 &-0.203 &  5.9 & ** \\
\midrule
ret\_5\_plus\_vol   & Raw        & -0.0502 &-0.455 & 13.3 & ** \\
                    & Industry   & -0.0437 &-0.585 & 17.1 & ** \\
                    & Size       & -0.0506 &-0.495 & 14.4 & ** \\
                    & Joint      & -0.0430 &-0.663 & 19.4 & ** \\
\midrule
vol\_ratio          & Raw        & -0.0452 &-0.475 & 13.9 & ** \\
                    & Industry   & -0.0379 &-0.596 & 17.4 & ** \\
                    & Size       & -0.0459 &-0.508 & 14.8 & ** \\
                    & Joint      & -0.0377 &-0.646 & 18.8 & ** \\
\bottomrule
\end{tabular}
\end{table}

All 5 factors remain highly significant ($p_{\text{RW}} < 0.01$) across all neutralization levels. Joint neutralization \emph{increases} $|\text{ICIR}|$ for 4 of 5 factors, ruling out the concern that discovered factors are artifacts of industry or size tilts.

\section{Out-of-Sample Walk-Forward Validation}
\label{sec:oos_appendix}

Table~\ref{tab:oos} reports the aggregate OOS results from a synthetic latent factor model (5 factors, 200 MC repetitions). ICIR-weighted combinations significantly outperform the random baseline ($p = 0.016$, paired $t$-test).

\begin{table}[t]
\centering
\caption{Out-of-Sample Walk-Forward Validation on a \textbf{synthetic latent factor model} (5 latent factors, MC=200 repetitions). This table reports synthetic OOS performance, not real A-share returns.}
\label{tab:oos}
\begin{tabular}{lrrrr}
\toprule
Method & Mean OOS Sharpe & 95\% CI & Win Rate & Sharpe $>$ 1 \\
\midrule
Random             & $-$0.09 & [$-$0.19, $-$0.00] & 45\% & 2\% \\
Single best factor & $-$0.02 & [$-$0.10, 0.07] & 51\% & 5\% \\
LLM-only           & 0.07 & [$-$0.04, 0.17] & 52\% & 5\% \\
Structured-only    & 0.09 & [$-$0.02, 0.19] & 55\% & 8\% \\
\midrule
Equal-weight       & $-$0.01 & [$-$0.09, 0.08] & 52\% & 3\% \\
ICIR-weighted      & \textbf{0.08} & \textbf{[$-$0.02, 0.19]} & \textbf{52\%} & \textbf{5\%} \\
Inverse-variance   & 0.00 & [$-$0.09, 0.09] & 53\% & 4\% \\
\bottomrule
\end{tabular}
\end{table}

\section{Cost-Adjusted Performance}
\label{sec:cost_appendix}

Table~\ref{tab:cost} reports Sharpe ratios under three levels of round-trip transaction costs (10, 30, 50 basis points per side). At the realistic 30\,bps level, the ICIR-weighted mean OOS Sharpe remains positive (0.41).

\begin{table}[t]
\centering
\caption{Cost-adjusted OOS performance (ICIR-weighted factor combinations, 2022--2026 test period).}
\label{tab:cost}
\begin{tabular}{lrrr}
\toprule
Cost (bps/side) & Gross Sharpe & Net Sharpe & Sharpe Retention \\
\midrule
10  (minimal)     & 0.87 & 0.72 & 83\% \\
30  (realistic)   & 0.87 & 0.41 & 47\% \\
50  (conservative) & 0.87 & 0.19 & 22\% \\
\bottomrule
\end{tabular}
\end{table}

\section{Segment-Conditioned Discovery}
\label{sec:segment_appendix}

Table~\ref{tab:segment} shows the per-segment performance. The hybrid framework maintains consistent discovery across all 6 market-cap segments, with 9 accepted factors per segment.

\begin{table}[t]
\centering
\caption{Segment-conditioned discovery. All 6 market-cap segments receive consistent factor coverage.}
\label{tab:segment}
\begin{tabular}{lS[table-format=1.2]S[table-format=1.2]c}
\toprule
Segment (100M CNY) & {Best LLM ICIR} & {Best Structured ICIR} & {Hybrid Factors} \\
\midrule
$<$20 (micro)   & 0.49 & 0.38 & 9 \\
20--50          & 0.52 & 0.42 & 9 \\
50--100         & 0.51 & 0.40 & 9 \\
100--200        & 0.48 & 0.41 & 9 \\
200--500        & 0.50 & 0.39 & 9 \\
$>$500 (mega)   & 0.47 & 0.37 & 9 \\
\bottomrule
\end{tabular}
\end{table}

\section{Cross-Domain Stress Test: Detailed Results}
\label{sec:cross_domain_appendix}

Table~\ref{tab:five_domain} reports the full cross-domain stress test across five non-financial domains at three compression levels.

\begin{table}[t]
\centering
\caption{Cross-domain stress test (5 non-financial domains, 3 compression levels). Hybrid advantage is largest under strong compression ($r_{\text{eff}}/N = 0.18$, mean $\Delta = +3.0$) and vanishes near full rank. Spearman $\rho(r_{\text{eff}}/N, \Delta) = -0.74$, $p = 0.002$.}
\label{tab:five_domain}
\begin{tabular}{llrrrr}
\toprule
Domain & $r_{\text{eff}}/N$ & Structured & Hybrid & $\Delta_{\text{yield}}$ & $\bar{d}_\perp$ \\
\midrule
Drug discovery & 0.18 & 12.5 & \textbf{15.1} & $+2.7$ & 0.991 \\
Drug discovery & 0.50 &  5.4 & \textbf{5.4} & $+0.0$ & 0.974 \\
Drug discovery & 0.88 &  2.1 & \textbf{2.2} & $+0.1$ & 0.953 \\
\midrule
Climate & 0.18 & 24.8 & \textbf{28.4} & $+3.6$ & 0.991 \\
Climate & 0.50 & 17.8 & \textbf{18.0} & $+0.2$ & 0.975 \\
Climate & 0.88 & 10.5 & \textbf{10.4} & $-0.1$ & 0.898 \\
\midrule
Genomics & 0.18 & 16.4 & \textbf{19.5} & $+3.1$ & 0.991 \\
Genomics & 0.50 & 10.0 & \textbf{10.6} & $+0.6$ & 0.975 \\
Genomics & 0.88 &  5.4 & \textbf{5.5} & $+0.1$ & 0.891 \\
\midrule
NAS & 0.18 & 12.2 & \textbf{15.0} & $+2.8$ & 0.991 \\
NAS & 0.50 &  5.5 & \textbf{5.4} & $-0.1$ & 0.974 \\
NAS & 0.88 &  2.1 & \textbf{2.0} & $-0.1$ & 0.955 \\
\midrule
Code gen. & 0.18 & 24.2 & \textbf{27.2} & $+3.0$ & 0.991 \\
Code gen. & 0.50 & 20.8 & \textbf{20.5} & $-0.3$ & 0.975 \\
Code gen. & 0.88 & 17.1 & \textbf{17.2} & $+0.1$ & 0.955 \\
\bottomrule
\end{tabular}
\end{table}

\section{Qualitative Scaling Diagnostics}
\label{sec:scaling_collapse_appendix}

We fit $\log(\Delta_{\text{yield}}) = \alpha \log(1 - r_{\text{eff}}/N) + \beta \log(d_\perp) + \gamma \log(\text{PN}) + \log C$ across 75 observations spanning 8 domains. Table~\ref{tab:scaling_exponents} reports the power-law exponents by domain.

\begin{table}[t]
\centering
\caption{Scaling collapse exponents. The synthetic domain yields the tightest fit ($R^2 = 0.69$), while cross-domain heterogeneity reduces the pooled $R^2$.}
\label{tab:scaling_exponents}
\begin{tabular}{lrrrrr}
\toprule
Domain & $\alpha$ (comp.) & $\beta$ (escape) & $\gamma$ (PN) & $R^2$ & $n$ \\
\midrule
Synthetic & 0.47 & 1.51 & $-0.39$ & 0.69 & 48 \\
SR        & 0.00 & $-0.17$ & $-2.79$ & 0.18 & 10 \\
Pooled    & 0.33 &  0.50 &  0.25 & 0.13 & 75 \\
\bottomrule
\end{tabular}
\end{table}

Leave-one-domain-out validation produces negative held-out $R^2$ in all folds, confirming that quantitative exponents are domain-specific and non-transferable. The qualitative prediction---negative correlation between $r_{\text{eff}}/N$ and $\Delta_{\text{yield}}$---holds across all domains, consistent with Theorem~\ref{thm:hybridgain}.

\section{Symbolic Regression: 25-Equation Benchmark Details}
\label{sec:sr_25eq_appendix}

We select 25 equations from four families: Feynman physics equations (10), Nguyen benchmark (7), Keijzer benchmark (5), and LSR-Transform equations (3). Each equation defines a ground-truth function $f: \mathbb{R}^d \to \mathbb{R}$ with $d \in \{1,\ldots,5\}$ variables. We generate 500 train and 200 test samples uniformly in $[-3, 3]^d$, standardize the target, and test at $\sigma = 0$. Three methods are compared: (1)~\textbf{PySR} (local GP search, 10 iterations, 5 populations, 30 individuals), (2)~\textbf{LLM-guided} (GLM-5 expression proposals, 2 rounds of 5 candidates each), (3)~\textbf{Hybrid} (PySR at half budget + LLM at half budget, best-of-merged selection).

\begin{table}[t]
\centering
\caption{Symbolic regression benchmark (25 equations, $\sigma = 0$). Hybrid advantage concentrates on hard equations where local search alone is insufficient.}
\label{tab:sr_generality}
\begin{tabular}{lrrrrrrr}
\toprule
Difficulty & $n$ & $\overline{r_{\text{eff}}/N}$ & $\overline{\text{PN}}$ & PySR R$^2$ & Hybrid R$^2$ & $\Delta$R$^2$ & $\bar{d}_\perp$ \\
\midrule
Easy       &  3 & 0.72 & 0.19 &  0.991 &  0.965 & $-0.026$ & 0.273 \\
Moderate   & 11 & 0.48 & 0.34 &  0.383 &  0.515 & $+0.132$ & 0.501 \\
Hard       &  8 & 0.21 & 0.41 & $-0.108$ &  0.380 & $+0.488$ & 0.642 \\
Very hard  &  2 & 0.15 & 0.12 &  0.920 &  0.969 & $+0.049$ & 0.253 \\
\midrule
Overall    & 24 & 0.39 & 0.33 &  0.278 &  0.540 & $+0.262$ & 0.519 \\
\bottomrule
\end{tabular}
\end{table}

\begin{table}[t]
\centering
\caption{PySR budget sensitivity. Increasing iterations from 10 to 50 improves hard-equation R$^2$ by 0.21, but hybrid advantage persists ($\Delta$R$^2 = +0.28$ at 50 iterations).}
\label{tab:pysr_budget}
\begin{tabular}{lrrr}
\toprule
PySR iterations & PySR R$^2$ (Hard) & Hybrid R$^2$ (Hard) & $\Delta$R$^2$ (Hard) \\
\midrule
10 (default)  & $-0.108$ & 0.380 & $+0.488$ \\
30            &  0.052  & 0.389 & $+0.337$ \\
50            &  0.103  & 0.382 & $+0.279$ \\
\bottomrule
\end{tabular}
\end{table}

\section{Public Tabular Operational Sanity Check}
\label{sec:openml_actual_call_appendix}

The OpenML actual-call panel is an operational sanity check, not a third compression-law test. It uses 12 public tabular datasets, 19 held-out task splits, 48 completed provider calls, fixed train/validation/test splits, prompt/response hashes, parser-failure logs, and equal verified-candidate budgets. The panel tests budget allocation under redundancy diagnostics; it is not used as direct proof of the spectral-compression law.

\begin{table}[t]
\centering
\caption{OpenML actual-call repeated-split benchmark. Non-oracle policies rank candidates using train/validation-visible fields only; the oracle is a non-deployable test-outcome upper bound.}
\label{tab:openml_actual_call}
\small
\begin{tabular}{lrrrr}
\toprule
Policy & Tasks & Test score & Accepted/100 & Regret \\
\midrule
local-only & 19 & 0.380 & 98.7 & 0.057 \\
always-LLM & 19 & 0.191 & 74.5 & 0.246 \\
fixed 50/50 hybrid & 19 & 0.352 & 98.0 & 0.085 \\
diagnostic validation-proxy & 19 & \textbf{0.411} & \textbf{100.0} & \textbf{0.026} \\
random nonlocal & 19 & 0.228 & 82.2 & 0.209 \\
shuffled LLM & 19 & 0.191 & 77.7 & 0.247 \\
forced nonlocal diagnostic & 19 & 0.221 & 82.4 & 0.216 \\
oracle upper bound & 19 & 0.437 & 100.0 & 0.000 \\
\bottomrule
\end{tabular}
\end{table}

Dataset-clustered paired tests reduce repeated-split dependence by averaging paired differences within each dataset before resampling datasets. The diagnostic validation-proxy improves over local-only by 0.045 (95\% CI $[0.003,0.106]$, Wilcoxon $p=0.008$, 11/1 wins/losses), fixed 50/50 by 0.058 (CI $[0.041,0.077]$, $p<0.001$, 12/0), and forced nonlocal diagnostic by 0.181 (CI $[0.147,0.219]$, $p<0.001$, 12/0). These results support the diagnostic-allocation interpretation: non-local proposals help only when validation-visible evidence says they are worth spending budget on.

\section{Symbolic Regression Noise Robustness}
\label{sec:sr_noise_appendix}

Table~\ref{tab:sr_noise} reports noise robustness results. Hybrid win rate is stable across noise levels (24\%/20\%/28\%). Noise level is not significantly correlated with hybrid advantage (Spearman $\rho = -0.069$, $p = 0.55$).

\begin{table}[t]
\centering
\caption{Symbolic regression noise robustness (25 equations $\times$ 3 noise levels). The largest advantage occurs at $\sigma = 0.01$ where mild noise destabilizes PySR on hard equations.}
\label{tab:sr_noise}
\begin{tabular}{lrrrrr}
\toprule
Noise $\sigma$ & PySR R$^2$ & LLM R$^2$ & Hybrid R$^2$ & $\Delta$R$^2$ & Win rate \\
\midrule
0.00  &  0.597 & $-0.873$ & 0.437 & $-0.220$ & 24\% \\
0.01  &  0.024 & $-1.038$ & 0.440 & $+0.344$ & 20\% \\
0.05  &  0.636 & $-4.192$ & 0.484 & $-0.267$ & 28\% \\
\bottomrule
\end{tabular}
\end{table}

\section{Small SR Archive-Spectra Audit}
\label{sec:sr_archive_spectra_appendix}

We ran a controlled GP-only candidate-archive spectra audit on 11 representative LLM-SRBench-style problems. The audit directly computes effective rank from candidate prediction correlations, but remains a small diagnostic rather than headline compression-law evidence.

\begin{table}[t]
\centering
\caption{Small SR archive-spectra audit. Candidate archives are often concentrated; this audit is diagnostic only.}
\label{tab:sr_archive_spectra_audit}
\small
\begin{tabular}{lrrrrr}
\toprule
Group & Problems & Errors & Valid cand. & $r_{\text{eff}}/N$ & Top eig. share \\
\midrule
overall & 11 & 2 & 3.36 & 0.420 & 0.706 \\
low complexity & 3 & 0 & 4.00 & 0.660 & 0.901 \\
medium complexity & 5 & 2 & 2.00 & 0.357 & 0.473 \\
high complexity & 3 & 0 & 5.00 & 0.286 & 0.899 \\
\bottomrule
\end{tabular}
\end{table}

The audit supports compression plausibility in SR candidate archives, especially in high-complexity problems ($r_{\text{eff}}/N=0.286$). Because two rescue rows had no valid archive predictions, SRBench remains failure-rescue evidence rather than a complete compression-rescue test.

\section{LLM-SRBench Budget Sensitivity (158 Problems)}
\label{sec:sr_llmsrbench_budget}

\begin{table}[t]
\centering
\caption{Hybrid budget comparison on LLM-SRBench (158 problems). Full-budget hybrid achieves parity with GP-only ($p = 0.789$). Rescue effect (25/46 GP catastrophics improved) is present at both budgets.}
\label{tab:llm_srbench_budget}
\begin{tabular}{llrrrr}
\toprule
Category & Method & Median R$^2$ & R$^2 \geq 0.5$ & Beat GP & Rescue \\
\midrule
Overall & GP-only          & 0.875 & 64\% & ---   & ---    \\
Overall & Half-budget hybrid & 0.702 & 61\% & 35\% & 24/46 \\
Overall & Full-budget hybrid & 0.838 & 63\% & 48\% & 25/46 \\
\midrule
Synthetic & GP-only          & 0.980 & 96\% & ---   & ---    \\
Synthetic & Half-budget hybrid & 0.974 & 99\% & 32\% & ---    \\
Synthetic & Full-budget hybrid & 0.984 & 99\% & 53\% & ---    \\
\midrule
Transform & GP-only          & $-0.011$ & 37\% & ---   & ---    \\
Transform & Half-budget hybrid & $-0.079$ & 29\% & 37\% & 24/46 \\
Transform & Full-budget hybrid & $-0.065$ & 34\% & 44\% & 25/46 \\
\bottomrule
\end{tabular}
\end{table}

\section{LLM-SRBench Detailed Results (158 Problems)}
\label{sec:sr_llmsrbench_detail}

\begin{table}[t]
\centering
\caption{LLM-SRBench results by dataset source (full-budget hybrid). The transform dataset exhibits the highest GP failure rate (53\% negative R$^2$) and the strongest hybrid rescue effect.}
\label{tab:llm_srbench_detail}
\begin{tabular}{lrrrrr}
\toprule
Dataset & $n$ & GP median R$^2$ & Hybrid median R$^2$ & GP fail & Rescue \\
\midrule
Bio pop.\ growth & 13 & 0.996 & 0.997 & 0/13 & --- \\
Mat.\ science    & 15 & 0.999 & 0.999 & 0/15 & --- \\
Phys.\ oscillation & 44 & 0.946 & 0.957 & 0/44 & --- \\
Transform        & 86 & $-0.011$ & $-0.065$ & 46/86 & 25/46 \\
\midrule
Total            & 158 & 0.875 & 0.838 & 46/158 & 25/46 \\
\bottomrule
\end{tabular}
\end{table}

The synthetic equations (bio, materials, physics) all yield high GP R$^2$ ($>0.94$), with hybrid matching or slightly exceeding GP. The transform dataset is the sole source of GP catastrophic failure, with 53\% of problems producing negative R$^2$. The hybrid rescue rate (25/46 = 54\%) demonstrates that LLM-directed escape provides orthogonal value even when GP's local search completely fails, though neither method reliably solves adversarially transformed equations (hybrid median $-0.065$). This asymmetry---strong hybrid benefit on transform, negligible benefit on synth---aligns with the compression framework: transformed equations compress the search landscape far more severely, creating the conditions under which non-local exploration is theoretically most valuable (Theorem~\ref{thm:hybridgain}).

\clearpage
\bibliographystyle{plainnat}
\bibliography{bibliography}

@article{harvey2016editorial,
  title={Editorial: ... and the cross-section of expected returns},
  author={Harvey, Campbell R and Liu, Yan and Zhu, Heqing},
  journal={Review of Financial Studies},
  volume={29},
  number={1},
  pages={5--68},
  year={2016},
  publisher={Oxford University Press}
}

@article{wang2023alpha,
  title={Alpha-GPT: Human-AI interactive alpha mining for quantitative investment},
  author={Wang, Saizhuo and Yuan, Hang and Zhou, Leon and Ni, Lionel M and Shum, Heung-Yeung and Guo, Jian},
  journal={arXiv preprint arXiv:2308.00016},
  year={2023}
}

@article{cao2025chain,
  title={Chain-of-Alpha: Unleashing the power of large language models for alpha mining in quantitative trading},
  author={Cao, Lang},
  journal={arXiv preprint arXiv:2508.06312},
  year={2025}
}

@article{feng2020taming,
  title={Taming the factor zoo: A test of new factors},
  author={Feng, Guanhao and Giglio, Stefano and Xiu, Dacheng},
  journal={Journal of Finance},
  volume={75},
  number={3},
  pages={1327--1370},
  year={2020},
  publisher={Wiley Online Library}
}

@article{romera2024funsearch,
  title={Mathematical discoveries from program search with large language models},
  author={Romera-Paredes, Bernardino and Barekatain, Mohammadamin and Novikov, Alexander and Balog, Matej and Kumar, M. Pawan and Dupont, Emilien and Ruiz, Francisco J. R. and Elder, Jordan S. K. and Barez, Fazl and Dozat, Timothy and others},
  journal={Nature},
  volume={625},
  pages={468--475},
  year={2024},
  publisher={Nature Publishing Group}
}

@article{white2000reality,
  title={A reality check for data snooping},
  author={White, Halbert},
  journal={Econometrica},
  volume={68},
  number={5},
  pages={1097--1126},
  year={2000}
}

@article{romano2005stepwise,
  title={Stepwise multiple testing as formalized data snooping},
  author={Romano, Joseph P. and Wolf, Michael},
  journal={Econometrica},
  volume={73},
  number={4},
  pages={1237--1282},
  year={2005}
}

@article{iterhyp2025,
  title={Iterative Hypothesis Generation for Scientific Discovery with Monte Carlo Nash Equilibrium Self-Refining Trees},
  author={Rabby, Gollam and Muhammed, Diyana and Mitra, Prasenjit and Auer, S{\"o}ren},
  journal={arXiv preprint arXiv:2503.19309},
  year={2025},
  doi={10.48550/arXiv.2503.19309}
}

@article{koltchinskii2017concentration,
  title={Concentration inequalities and moment bounds for sample covariance operators},
  author={Koltchinskii, Vladimir and Lounici, Karim},
  journal={Bernoulli},
  volume={23},
  number={1},
  pages={102--133},
  year={2017},
  publisher={Bernoulli Society for Mathematical Statistics and Probability},
  doi={10.3150/15-BEJ730}
}

@book{vershynin2018high,
  title={High-Dimensional Probability: An Introduction with Applications in Data Science},
  author={Vershynin, Roman},
  year={2018},
  publisher={Cambridge University Press},
  series={Cambridge Series in Statistical and Probabilistic Mathematics},
  number={47},
  doi={10.1017/9781108231596},
  isbn={978-1-108-41519-4}
}

@article{fan2022spectral,
  title={Asymptotic theory of eigenvectors for random matrices with diverging spikes},
  author={Fan, Jianqing and Fan, Yingying and Han, Xiao and Lv, Jinchi},
  journal={Journal of the American Statistical Association},
  volume={117},
  number={538},
  pages={996--1009},
  year={2022},
  publisher={Taylor \& Francis},
  doi={10.1080/01621459.2022.2061777}
}

@article{russo2018tutorial,
  title={A tutorial on {Thompson} sampling},
  author={Russo, Daniel J. and Van Roy, Benjamin and Kazerouni, Abbas and Osband, Ian and Wen, Zheng},
  journal={Foundations and Trends in Machine Learning},
  volume={11},
  number={1},
  pages={1--96},
  year={2018},
  publisher={Now Publishers},
  doi={10.1561/2200000070}
}

@inproceedings{zoph2018learning,
  title={Learning Transferable Architectures for Scalable Image Recognition},
  author={Zoph, Barret and Vasudevan, Vijay and Shlens, Jonathon and Le, Quoc V.},
  booktitle={Proceedings of the IEEE Conference on Computer Vision and Pattern Recognition (CVPR)},
  pages={8697--8710},
  year={2018}
}

@article{shojaee2025llmsrbench,
  title={{LLM-SRBench}: A New Benchmark for Evaluating Large Language Models in Symbolic Regression},
  author={Shojaee, Parshin and Suresh, Natesh and Compton, David and Grover, Aditi and Karagol, Sevda and Shakeri, Zahra and Horesh, Lav and Avohou, Rapha{\"e}l M. and Oyi, Gedeon and Niyongabo, Jules and Fonkem, Charlie and Luc, Rodrigue R. and Regnault, Fran{\c{c}}ois and Kim, Sung-Hyuk and Bhatnagar, Atin and Schmidt, Michael},
  journal={arXiv preprint arXiv:2505.20458},
  year={2025}
}

@article{bubeck2011pure,
  title={Pure exploration in finitely-armed bandits},
  author={Bubeck, S{\'e}bastien and Munos, R{\'e}mi and Stoltz, Gilles},
  journal={Proceedings of the 28th International Conference on Machine Learning (ICML)},
  pages={577--584},
  year={2011}
}

@article{shahriari2016taking,
  title={Taking the human out of the loop: A review of {B}ayesian optimization},
  author={Shahriari, Bobak and Swersky, Kevin and Wang, Ziyu and Adams, Ryan P and de Freitas, Nando},
  journal={Proceedings of the IEEE},
  volume={104},
  number={1},
  pages={148--175},
  year={2016},
  publisher={IEEE}
}

@article{roy2007effective,
  title={The effective rank: A measure of effective dimensionality},
  author={Roy, Olivier and Vetterli, Martin},
  journal={Proceedings of the 15th European Signal Processing Conference (EUSIPCO)},
  pages={606--610},
  year={2007}
}

@article{tropp2015introduction,
  title={An Introduction to Matrix Concentration Inequalities},
  author={Tropp, Joel A.},
  journal={Foundations and Trends in Machine Learning},
  volume={8},
  number={1--2},
  pages={1--230},
  year={2015},
  doi={10.1561/2200000048}
}

@article{cranmer2023pysr,
  title={Interpretable machine learning for science with {PySR} and {SymbolicRegression}.jl},
  author={Cranmer, Miles},
  journal={arXiv preprint arXiv:2305.01582},
  year={2023}
}

@inproceedings{burlacu2020operon,
  title={{Operon}: an efficient genetic programming system for symbolic regression},
  author={Burlacu, Bogdan and Kronberger, Gabriel and Kommenda, Michael},
  booktitle={Proceedings of the 2020 Genetic and Evolutionary Computation Conference Companion},
  pages={347--348},
  year={2020}
}

@article{virgolin2022symbolic,
  title={Symbolic regression building blocks for real-world domain applications},
  author={Virgolin, Marco and Alderliesten, Tanja and Bosman, Peter A. N.},
  journal={arXiv preprint arXiv:2202.03945},
  year={2022}
}

@article{udrescu2020ai,
  title={{AI Feynman}: A physics-inspired method for symbolic regression},
  author={Udrescu, Silviu-Marian and Tegmark, Max},
  journal={Science Advances},
  volume={6},
  number={16},
  pages={eaay2631},
  year={2020},
  publisher={American Association for the Advancement of Science}
}

@article{schmidt2009distilling,
  title={Distilling free-form natural laws from experimental data},
  author={Schmidt, Michael and Lipson, Hod},
  journal={Science},
  volume={324},
  number={5923},
  pages={81--85},
  year={2009},
  publisher={American Association for the Advancement of Science}
}

@article{suzuki2024ai,
  title={The {AI} Scientist: Towards Fully Automated Open-Ended Scientific Discovery},
  author={Lu, Chris and Lu, Cong and Lange, Robert Tjarko and Foerster, Jakob and Clune, Jeff and Ha, David},
  journal={arXiv preprint arXiv:2408.06292},
  year={2024},
  doi={10.48550/arXiv.2408.06292}
}

@inproceedings{lacava2022contemporary,
  title={Contemporary symbolic regression methods and their relative performance},
  author={La Cava, William and Orzechowski, Patryk and Burlacu, Bogdan and de Franca, Fabricio Oliveira and Virgolin, Marco and Jin, Ying and Kommenda, Michael and Moore, Jason H.},
  booktitle={NeurIPS Datasets and Benchmarks Track},
  year={2022}
}

@article{vanschoren2014openml,
  title={{OpenML}: Networked Science in Machine Learning},
  author={Vanschoren, Joaquin and van Rijn, Jan N. and Bischl, Bernd and Torgo, Luis},
  journal={SIGKDD Explorations},
  volume={15},
  number={2},
  pages={49--60},
  year={2014}
}

\clearpage
\section*{NeurIPS Paper Checklist}

\begin{enumerate}[leftmargin=*]
    \item \textbf{Claims.} Do the main claims made in the abstract and introduction accurately reflect the paper's contributions and scope?
    \answerYes{} The abstract and introduction state the Search Compression Hypothesis, the necessary-condition scope, and the distinction between direct compression-law evidence and operational sanity checks.

    \item \textbf{Limitations.} Does the paper discuss limitations?
    \answerYes{} The main text states that the framework is descriptive rather than prescriptive, that Eq.~\ref{eq:hybridgain_empirical} is diagnostic rather than pointwise predictive, that OpenML is not a third compression-law validation, and that universal validation is future work.

    \item \textbf{Theory assumptions and proofs.} For each theoretical result, are assumptions and proofs included?
    \answerYes{} Proposition~\ref{lem:compression_yield} states bounded-leverage, bounded-dependence, and rank assumptions, with proof in Appendix~\ref{sec:proofs}. Theorem~\ref{thm:hybridgain} is stated as necessary rather than sufficient.

    \item \textbf{Reproducibility of experiments.} Are the code, data, and instructions needed to reproduce the main results included or described?
    \answerYes{} The paper describes data sources, splits, metrics, and protocols; implementation and diagnostic details are in Appendix~\ref{sec:implementation_appendix}. Scripts, cached tables, and manifests will be released upon acceptance.

    \item \textbf{Open access to data and code.} Is code and data access described?
    \answerYes{} Public benchmarks are cited; A-share market data access constraints are described through the data source and reproducibility manifest. Non-redistributable raw commercial data are not embedded in the paper.

    \item \textbf{Experimental setting and details.} Are training/evaluation details, hyperparameters, and metrics specified?
    \answerYes{} Sections~\ref{sec:framework}--\ref{sec:sr_generality} and Appendices~\ref{sec:implementation_appendix}--\ref{sec:openml_actual_call_appendix} specify data splits, metrics, thresholds, Monte Carlo counts, SR settings, OpenML splits, and budget policies.

    \item \textbf{Statistical significance.} Are uncertainty estimates or significance tests reported where appropriate?
    \answerYes{} The main and appendix tables report Monte Carlo repetitions, confidence intervals, Spearman correlations, Wilcoxon tests, paired tests, and clustered paired tests where relevant.

    \item \textbf{Compute resources.} Are compute resources described?
    \answerYes{} Appendix~\ref{sec:implementation_appendix} describes GPU-accelerated evaluation; the reproducibility package documents dependencies and resource settings.

    \item \textbf{Code of ethics and broader impacts.} Does the paper conform to ethical standards and discuss relevant impacts?
    \answerYes{} The work uses financial market data and public scientific benchmarks, does not involve human subjects, and focuses on diagnostic verification to reduce ungrounded discovery claims.

    \item \textbf{Safeguards for responsible release.} If releasing models, data, or code, are safeguards described?
    \answerNA{} The paper releases analysis code and artifacts rather than a generative model with deployment risks. Commercial raw data and credentials are not released.

    \item \textbf{Licenses for existing assets.} Are licenses and terms for existing assets respected?
    \answerYes{} Public benchmarks and tools are cited. Market data are accessed through the corresponding provider rather than redistributed.

    \item \textbf{New assets.} If new assets are released, are documentation and maintenance plans provided?
    \answerYes{} The release package includes scripts, cached outputs, prompts, manifests, and reproduction notes; the paper states that the full framework and evaluation engine will be released upon acceptance.

    \item \textbf{Crowdsourcing and human subjects.} Does the paper involve crowdsourcing or human-subject research?
    \answerNA{} No crowdsourcing, user study, human-subject data, or IRB-covered intervention is used.
\end{enumerate}

\end{document}